%% file: collas2022_conference.tex
\documentclass{article} 
\usepackage{collas2022_conference,times}

\input{math_commands.tex}

\usepackage{hyperref}
\hypersetup{
    colorlinks=true,
    linkcolor=red,
    filecolor=magenta,      
    urlcolor=blue,
    citecolor=purple,
    pdftitle={Overleaf Example},
    pdfpagemode=FullScreen,
    }

\usepackage[noend]{algorithm,algorithmic}
\usepackage{caption}
\usepackage{amsmath, amssymb}
\usepackage[utf8x]{inputenc} 
\usepackage{hyperref}
\usepackage{interval}
\usepackage{caption}
\usepackage{hyperref}
\usepackage{cleveref}%
\usepackage{amsthm}

\usepackage{tikz}

\definecolor{darkgray}{rgb}{0.50, 0.50, 0.50}
\definecolor{gray}{rgb}{0.70, 0.70, 0.70}
\definecolor{lightgray}{rgb}{0.92, 0.92, 0.92}

\newtheorem{theorem}{Theorem}
\newtheorem{corollary}[theorem]{Corollary}

\newtheorem{definition}{Definition}

\title{Lifelong Matrix Completion with Sparsity-Number}

\author{Ilqar Ramazanli \\
Carnegie Mellon University\\
Pittsburgh, USA \\
\texttt{iramazan@alumni.cmu.edu}  
}

 \collasfinalcopy 

\begin{document}

\maketitle

\begin{abstract}
Matrix completion problem has been previously studied under various adaptive and passive settings.
Previously, researchers have proposed passive, two-phase and single-phase algorithms using coherence parameter, and multi phase algorithm using sparsity-number.
It has been shown that the method using sparsity-number reaching to theoretical lower bounds in many conditions.
However, the aforementioned method is running in many phases through the matrix completion process, therefore it makes much more informative decision at each stage.
Hence, it is natural that the method outperforms previous algorithms.
In this paper, we are using the idea of sparsity-number and propose and single-phase column space recovery algorithm which can be extended to two-phase exact matrix completion algorithm. 
Moreover, we show that these methods are as efficient as multi-phase matrix recovery algorithm.
We provide experimental evidence to illustrate the performance of our algorithm.
\end{abstract}

\vspace{8mm}
\section{Introduction}

Lifelong Adaptive Matrix Completion problem has been the center of the attention of research for many recent years. 
Previously matrix completion has been mainly studied in passive sampling setting by uniformly sampling and observing some subset of entries of the matrix in the beginning of the algorithm and applying nuclear norm minimization.
In a series of work in  \cite{gross}, \cite{recht1} and \cite{recht2} it has been shown that uniformly sampling and observing $\Omega ((m + n)r \mathrm{max} (\mu_0^2, \mu_1^2)\log^2{n_2})$ many entries are enough to recover an $m \times n$ matrix of rank $r$  with $\mu_0$ and $\mu_1$ column and row space coherence parameters using this idea.
Moreover, \cite{tao} showed this passive uniform sampling setting we need at least $\Omega(m r \mu_0 \log{n})$ observations to recover the matrix exactly. 
Therefore, this result was concluding almost optimality of the nuclear norm minimization where sampling method is passive.\\[2.9ex]
However, in modern data analysis it has been show that adaptive sensing and sampling methods are outperforming passive sampling methods \cite{ramazanli2022performance, haupt,kzmn, ramazanli2020adaptive}.
In a follow up work the paper \cite{akshay1} has proposed an adaptive matrix completion algorithm  that can actually recover underlying matrix using at most $\mathcal{O}(\mu_0 r^{1.5}\log{\frac{r}{\epsilon}})$ many observation.
Later authors further improved their existing results in \cite{akshay2} to show that the adaptive matrix completion algorithm can actually successfully perform by just observing 
$\mathcal{O}(\mu_0 r \log^2{\frac{r}{\epsilon}})$ many observation.\\[2.9ex]
The latest result has been even further optimized in the paper \cite{nina} to present an adaptive algorithm which performs only with observation complexity of just  $\mathcal{O}(\mu_0 r \log{\frac{r}{\epsilon}})$.
Recently \cite{poczos2020optimal} proposed a new parametrization method called sparsity-number and a multi-phase algorithm which shows that using this new method, the observation complexity can go as low as $\mathcal{O}((m+n-r)r)$.
The idea of this method to use the sparsest vector of column and row spaces instead of coherence in the recovery process.
Studying the sparsest vector of the column and row space has been focus in many research problems (\cite{sparse2}, \cite{sparse1}, \cite{sparse3}), and \cite{poczos2020optimal} showed how to use this parameter efficiently for adaptive matrix completion. 
In this paper, we propose a matrix completion method which merges ideas of single phase algorithms such as \cite{nina} with the idea of sparsity number and multi-phase method \cite{poczos2020optimal} which is resulting in a single phased algorithm which is as efficient as multi-phased one.\\[2.9ex]
Low-rank matrix completion plays a significant role in many real-world applications, including camera motion inferring, multi-class learning, positioning of sensors, and gene expression analysis \cite{ nina, akshay1}.
In gene expression analysis, the target matrix represents expression levels for various genes across several conditions.
Measuring gene expression, however, is expensive, and we would like to estimate the target matrix with a few observations as possible. In this paper, we provide an algorithm that can be used for matrix completion from limited data. 
Roughly speaking, to find each unknown expression level, we are supposed to do multiple measurements.
Each of the additional measurements has its extra cost. 
Naturally, we aim to solve the entire problem using the least possible measurement cost.

\vspace{6mm}

\section{Preliminaries}
\label{sec:preliminaries}

In this section we start by providing notations and definitions those are used throughout the paper. 
Then we will provide a single phase \cite{nina} and multi-phase \cite{opt2} matrix recovery algorithms.\\[2.9ex]
Throughout the paper, we denote by $\mathbf{M}$ the target underlying $m \times n$ sized rank-$r$  matrix that we want to recover.  
$\|x\|_{p}$ denote the $L_{p}$ norm of a vector $x \in \mathbb{R}^n$.
We call $x_i$ the $i$'th coordinate of $x$. 
For any, $\Omega \subset [n]$ let $x_{\Omega}$ denote the induced subvector of $x$ from coordinates $\Omega$.
For any $\mathbf{R}\subset[m]$, $\mathbf{M}_{\mathbf{R}:}$ stands for an $|\mathbf{R}| \times n$ sized submatrix of $\mathbf{M}$ that rows are restricted by $\mathbf{R}$. 
We define $\mathbf{M}_{:\mathbf{C}}$ in a similar way for restriction with respect to columns.
Intuitively, $\mathbf{M}_{\mathbf{R}:\mathbf{C}}$ defined for $|\mathbf{R}|\times |\mathbf{C}|$ sized submatrix of $\mathbf{M}$ with rows restricted to $\mathbf{R}$ and columns restriced to $\mathbf{C}$.
Moreover, for the special case $\mathbf{M}_{i:}$ stands for $i$-th row and $\mathbf{M}_{:j}$ stands for the $j$'th column.
Similarly, $\mathbf{M}_{i:\mathbf{C}}$ will represent the restriction of the row $i$ by $C$ and $\mathbf{M}_{\mathbf{R}:j}$ represents restriction of the column $j$ by $\mathbf{R}$.\\[2.9ex]
The following definitions are key parameters in this paper, first we define coherence which is an important parameter in the matrix completion research domain. 
Then, we visit the definition of the sparsity-number.\\[2.9ex]

\begin{definition}
Coherence parameter of a matrix $\mathbf{M}$ with column space $\mathbb{U}$ is defined as following where $\mathcal{P}_{\mathbb{U}}$ represents the orthogonal projection operator onto the subspace $\mathbb{U}$. 
\begin{align*}
\mu(\mathbb{U}) = \frac{n}{r} \underset{1 \leq j \leq n}{\max} || \mathcal{P}_{\mathbb{U}} e_j ||^2,
\end{align*}

\end{definition}
One can observe that if $e_j \in \mathbb{U} $ for some $j\in [n]$, then the coherence will attain its maximum value: $\mu(\mathbb{U}) = \frac{n}{r}$. \\
\begin{definition}
\textit{Nonsparsity-number} of a vector $x\in \mathbb{R}^m$ represented by $\psi(x)$ and defined as $\psi(x) = \|x \|_0$.
Moreover, \textit{nonsparsity-number} of an $m\times n$ sized matrix $\mathbf{M}$ and a subspace $\mathbb{U}\subseteq \mathbb{R}^m$ also represented by $\psi$ and defined as:
\begin{align*}
    \psi(\mathbf{M})=\mathrm{min}\{\psi(x) | x=\mathbf{M}z \text{ and } z\notin \mathrm{null}(\mathbf{M})  \} \hspace{15mm}  \psi(\mathbb{U})=\mathrm{min}\{\psi(x) | x\in \mathbb{U} \text{ and } x\neq 0 \}
\end{align*}

Sparsity-number is just completion of nonsparsity-number and is donated by $\overline{\psi}$:\\
\begin{align*}
\overline{\psi}(x) = m-\psi(x) \hspace{20mm}
\overline{\psi}(\mathbf{M})=  m- \psi(\mathbf{M}) \hspace{20mm}
\overline{\psi}(\mathbb{U})=  m- \psi(\mathbb{U}) 
\end{align*}  
\end{definition}

\vspace{6mm}
\subsection{Single Phase Matrix Completion} 
Here we provide the generic framework for single-phase matrix completion algorithm that has been used by \cite{akshay1, akshay2, nina}.
The algorithm uniformly sample $d$-many entry in each column and observe them. 
Based on the observation, if it has been decided the column is linearly independent with previous columns then algorithm decides to observe entire column. 
Initially, it has been shown that choosing $d = \mathcal{O}(n\mu_0 r^{1.5}\log{r})$ would be enough this algorithm to perform successfully in \cite{akshay1}.
Then, authors showed setting $d = \mathcal{O}(n\mu_0 r\log^2{r})$  would also be enough \cite{akshay2}.
In a following work \cite{nina} optimized the complexity by reducing one of the $\log{r}$ factors from $d$ and proving $d = \mathcal{O}(n\mu_0 r \log{r})$ still works successfully with probability $1-\epsilon$.

\begin{algorithm}
\caption*{  \hypertarget{hn2016}{\textbf{HN2016:}} Exact recovery  \cite{nina}.}
\textbf{Input:}   $d=\mathcal{O}(\mu_0 r \log{\frac{r}{\epsilon}})$\\
 \textbf{Initialize:}  $k=0 , \widehat{\mathbf{U}}^0 = \emptyset$ 
\begin{algorithmic}[1]
    \STATE Draw uniformly random entries $\Omega \subset [m]$ of size $d$   
    \FOR{$i$ from $1$ to $n$}        
    \STATE  \hspace{0.2in} \textbf{if} $\| \mathbf{M}_{\Omega:i}-{\mathcal{P}_\mathbf{\widehat{U}_{\Omega}^k}} \mathbf{M}_{\Omega:i}\| >0$ 
    \STATE \hspace{0.2in}  \hspace{0.2in} Fully observe $\mathbf{M}_{:i}$ 
    \STATE  \hspace{0.2in}  \hspace{0.2in}  $\widehat{\mathbf{U}}^{k+1} \leftarrow \widehat{\mathbf{U}}^{k} \cup \mathbf{M_{:i}} $, Orthogonalize $\widehat{\mathbf{U}}^{k+1}$,  $k=k+1$          
   \STATE   \hspace{0.2in}  \textbf{otherwise:} $\widehat{\mathbf{M}}_{:i} = \widehat{\mathbf{U}}^k {\widehat{\mathbf{U}}^{k^+}_{\Omega}}
 \widehat{\mathbf{M}}_{\Omega :i}$
\ENDFOR
\end{algorithmic}
\textbf{Output:}  $\widehat{\mathbf{M}}$
\end{algorithm}

\vspace{6mm}
\subsection{Multi Phase Matrix Completion} 
\vspace{3mm}

Here we provide multi-phase matrix completion algorithm from \cite{poczos2020optimal}.
The are two main differences between this algorithm and previous ones. 
First, previous algorithms goes through the columns of the matrix just once, however, here we go through each column many times.
Second, there is a difference in terms of available pre-information.
One can see that, in order to perform previous algorithms, we need to have at least estimated information of rank $r$ and column space coherence $\mu_0$. 
However, in order to perform $\mathbf{ERRE}$ all we need to know is either column or row space is not highly coherent.

\vspace{6mm}

\begin{algorithm}
\caption*{  \hypertarget{erre}{\textbf{ERRE:}}  Exact recovery while rank estimation. \cite{poczos2020optimal} }

 \textbf{Input:}  $T$-delay parameter at the end of algorithm \\
 \textbf{Initialize:} $R = \emptyset, C=\emptyset,  \widehat{r}= 0,  delay = 0$

\begin{algorithmic}[1]\label{alg1}

    \WHILE{$delay <T $}
    \STATE   $ delay= delay + 1$

    \FOR{$j$ from $1$ to $n$}
    \STATE Uniformly pick an unobserved entry $i$ from $\mathbf{M}_{:j}$  
    \STATE $\widehat{R}={R} \cup \{i\} , \widehat{C}= {C} \cup \{j\}$  
    
    \STATE \textbf{If} $ \mathbf{M}_{\widehat{R}:\widehat{C}}$ is nonsingular :
    \STATE \hspace{0.1in} Fully observe $\mathbf{M}_{:j}$ and $\mathbf{M}_{i:} $ 
and set  $R =\widehat{R}$ , $C = \widehat{C}$ , $\widehat{r}=\widehat{r}+1, delay=0$
    \ENDFOR

\ENDWHILE

\STATE Orthogonalize column vectors in $C$ and assign to $\mathbf{U}$

\FOR{each column $j \in [n]\setminus C$ }
\STATE $\widehat{\mathbf{M}}_{:j} = \widehat{\mathbf{U}} {\widehat{\mathbf{U}}_{R:}}^+ \widehat{\mathbf{M}}_{R: j}$
\ENDFOR

\end{algorithmic}
\textbf{Output:} $\widehat{\mathbf{M}}, \widehat{r}$
\end{algorithm}

\vspace{3mm}

The authors proved that the observation complexity could be bounded by the following expression which is more optimal than previous methods and many times it matches $\mathcal{O}((m+n-r)r)$ which is theoretical lower bound for any kind of rank-$r$ matrix recovery algorithm.\\

\begin{theorem}  \label{thm:erre}
Let $r$ be the rank of underlying $m\times n$ sized matrix $\mathbf{M}$ with column space $\mathbb{U}$ and row space $\mathbb{V}$. 
Then, $\mathbf{ERRE}$ exactly recovers the underlying matrix $\mathbf{M}$ while estimating rank with probability at least $1- ( \epsilon + e^{-T\frac{\psi(\mathbb{U})\psi(\mathbb{V})}{m}})$  using number of observations at most:
\vspace{3mm}
\begin{align*}
  (m+n-r)r+Tn +  \mathrm{min} \Big(  2 \frac{m n}{\psi(\mathbb{U})}\log{(\frac{r}{\epsilon})} , \frac{\frac{2m}{\psi(\mathbb{U})}(r+2 +\log{\frac{1}{\epsilon}})}{\psi(\mathbb{V})}n ) \Big)
\end{align*}
\end{theorem}

\vspace{3mm}

\section{Main Results}

In this section, we provide single-phase column space recovery algorithm using \textit{sparsity-number} which can be easily extended to two-phase matrix recovery algorithm.
The algorithm works in a similar flavor to \cite{nina} that it requires a preinformation regarding properties of the matrix in order to execute.\\[2.5ex]
However, there is one clear difference in terms of preinformation that, here we require estimated information for sparsity number rather than coherence, and also we use the information regarding row-space as well rather than simply columns space.
The reason behind this is, in deep mathemtical analysis of aforementioned methods shows us all these methods work efficiently under the condition row spaces are highly coherent and in-depth analysis of these methods tells us that these algorithms designed for the highest value of row space coherence, and they also work perfectly well for remaining cases.\\[2.5ex]
In the following algorithm, we extend these algorithms to a method that can enjoy the properties of row space; meanwhile, performing similarly good for highly coherent row space matrices. 
Note taht, the algorithm below designed and analyzed under the condition that estimated r, $\psi(\mathbb{U})$, and $\psi(\mathbb{V})$ provided.  
However, under the condition just estimated $r$ and $\mu_0$ provided, we can set $\psi(\mathbb{V})$ to be equal to 1 and use the inequality $\mu_0 r > \frac{m}{k}$  to transfer the information of $\mu_0$ to $\psi(\mathbb{U})$. 
Then the algorithm below will perform as good as \cite{nina} with observation complexity of $mr + \mathcal{O}(nr\mu_0 \log{\frac{r}{\epsilon}})$.

\vspace{5mm}

\begin{algorithm}
\caption*{  \hypertarget{erei}{\textbf{EREI:}}  Exact recovery with estimated information}
\textbf{Input:}   $r, \psi(\mathbb{U}), \psi(\mathbb{V})$\\
 \textbf{Initialize:}  $R=\emptyset, k=0 , \widehat{\mathbf{U}}^0 = \emptyset, d =  \mathrm{min}\Big(  2 \frac{m }{\psi(\mathbb{U})}\log{(\frac{r}{\epsilon})} , \frac{\frac{2m}{\psi(\mathbb{U}}(r+2 +\log{\frac{1}{\epsilon}})}{\psi(\mathbb{V})} \Big)$  
 
\begin{algorithmic}[1]
    \STATE Draw uniformly random entries $\Omega \subset [m]$ of size $d$   
    \FOR{$i$ from $1$ to $n$}        
    \STATE  \hspace{0.2in} \textbf{if} $\| \mathbf{M}_{\Omega:i}-{\mathcal{P}_\mathbf{\widehat{U}_{\Omega}^k}} \mathbf{M}_{\Omega:i}\| >0$ 
    \STATE \hspace{0.2in}  \hspace{0.2in} Fully observe $\mathbf{M}_{:i}$ 
    \STATE  \hspace{0.2in}  \hspace{0.2in}  $\widehat{\mathbf{U}}^{k+1} \leftarrow \widehat{\mathbf{U}}^{k} \cup \mathbf{M_{:i}} $, Orthogonalize  $\widehat{\mathbf{U}}^{k+1}$, $k=k+1$          
    \STATE  \hspace{0.2in}  \hspace{0.2in}  Select a row $a \in \Omega \setminus R$ that, $\mathbf{\widehat{U}_{  R\cup \{a\}}^{k+1}}$ is rank $k+1$ then
    $R \leftarrow R \cup \{a\}$
    \STATE \hspace{0.2in} Draw uniformly random entries $\Delta \subset [m]\setminus R$ of size $d$ and $\Omega = \Delta \cup R$   
\ENDFOR

    \STATE Observe entire $\mathbf{M}_{R:}$ 
    \FOR{$i$ from $1$ to $n$}        
       \STATE   \hspace{0.2in}  \textbf{if $\mathbf{M}_{:i}$} \text{ not fully observed \textbf{then} : } $\widehat{\mathbf{M}}_{:i} = \widehat{\mathbf{U}}^k {\widehat{\mathbf{U}}^{k^+}_{R:}}
 \widehat{\mathbf{M}}_{R :i}$
 \ENDFOR
\end{algorithmic}
\textbf{Output:}  $\widehat{\mathbf{M}}$x
\end{algorithm}

\vspace{5mm}

Below we show the execution of the algorithm for $d=2$ and $r=1$. \colorbox{gray}{$a$} stands for entries that is observed randomly,
\colorbox{darkgray}{$a$} stands for entries that is observed deterministically and 
\colorbox{lightgray}{$a$} stands for recovered entries without observing. 
Note that $\Omega_1 =\{1,5\}, \Omega_2 =\{2,5\}, \Omega_3 =\{1,3\}$ and after 3-rd iteration $R$ becomes $\{ 3 \}$ therefore, in fourth column $\mathbf{M}_{3:4}$ observed deterministically, besides together random observations $\Omega_4 = \{5,6\}$.
After all of the iterations completed, we observe the entire $\mathbf{M}_{3:}$, and in the next step, we recover remaining entries.

\vspace{5mm}

\begin{equation*}\label{eq:appendrow}
  \arraycolsep=0.3pt
\newcommand\x{\colorbox{white}}
\newcommand\y{\colorbox{gray}}
\newcommand\z{\colorbox{lightgray}}
\newcommand\w{\colorbox{darkgray}}
  \left[ 
  \begin{array}{cccc}
    \y{0}    & \x{0}   & \x{0}   & \x{0}  \\
    \x{0}    & \x{0}   & \x{0}   & \x{0} \\
    \x{1}    & \x{3}   & \x{2}   & \x{3} \\
    \x{0}    & \x{0}   & \x{0}   & \x{0}  \\
    \y{0}    & \x{0}   & \x{0}   & \x{0} \\
    \x{2}    & \x{6}   & \x{4}   & \x{6} \\
  \end{array} 
  \right]
   \left[ \begin{array}{cccc}
    \y{0}    & \x{0}   & \x{0}   & \x{0}  \\
    \x{0}    & \y{0}   & \x{0}   & \x{0} \\
    \x{1}    & \x{3}   & \x{2}   & \x{3} \\
    \x{0}    & \x{0}   & \x{0}   & \x{0}  \\
    \y{0}    & \y{0}   & \x{0}   & \x{0} \\
    \x{2}    & \x{6}   & \x{4}   & \x{6} \\
  \end{array} 
  \right]
  \left[\begin{array}{cccc}
    \y{0}    & \x{0}   & \y{0}   & \x{0}  \\
    \x{0}    & \y{0}   & \x{0}   & \x{0} \\
    \x{1}    & \x{3}   & \y{2}   & \x{3} \\
    \x{0}    & \x{0}   & \x{0}   & \x{0}  \\
    \y{0}    & \y{0}   & \x{0}   & \x{0} \\
    \x{2}    & \x{6}   & \x{4}   & \x{6} \\
  \end{array}\right]
  \left[\begin{array}{cccc}
    \y{0}    & \x{0}   & \w{0}   & \x{0}  \\
    \x{0}    & \y{0}   & \w{0}   & \x{0} \\
    \x{1}    & \x{3}   & \w{2}   & \w{3} \\
    \x{0}    & \x{0}   & \w{0}   & \x{0}  \\
    \y{0}    & \y{0}   & \w{0}   & \y{0} \\
    \x{2}    & \x{6}   & \w{4}   & \y{6} \\
  \end{array}\right]
  \left[\begin{array}{cccc}
    \y{0}    & \x{0}   & \w{0}   & \x{0}  \\
    \x{0}    & \y{0}   & \w{0}   & \x{0} \\
    \w{1}    & \w{3}   & \w{2}   & \w{3} \\
    \x{0}    & \x{0}   & \w{0}   & \x{0}  \\
    \y{0}    & \y{0}   & \w{0}   & \y{0} \\
    \x{2}    & \x{6}   & \w{4}   & \y{6} \\
  \end{array}\right]
  \left[\begin{array}{cccc}
    \y{0}    & \z{0}   & \w{0}   & \z{0}  \\
    \z{0}    & \y{0}   & \w{0}   & \z{0} \\
    \w{1}    & \w{3}   & \w{2}   & \w{3} \\
    \z{0}    & \z{0}   & \w{0}   & \z{0}  \\
    \y{0}    & \y{0}   & \w{0}   & \y{0} \\
    \z{2}    & \z{6}   & \w{4}   & \y{6} \\
  \end{array}\right]
\end{equation*}\\

\vspace{5mm}

\begin{theorem} \label{thm:erei}
Let $r$ be the rank of underlying $m\times n$ sized matrix $\mathbf{M}$ with column space $\mathbb{U}$ and row space $\mathbb{V}$. 
Then, with probability $1-\epsilon$
the algorithm $\mathbf{EREI}$ exactly recovers the underlying matrix $\mathbf{M}$ using number of observations at most
\begin{align*}
(m+n-r)r +  \mathrm{min}\Big(  2 \frac{m n}{\psi(\mathbb{U})}\log{(\frac{r}{\epsilon})} , \frac{\frac{2m}{\psi(\mathbb{U})}(r+2 +\log{\frac{1}{\epsilon}})}{\psi(\mathbb{V})}n \Big)  
\end{align*}
\end{theorem}

\begin{proof}

We show that the observation complexity is upper bounded both by $$(m+n-r)r+  2 \frac{m }{\psi(\mathbb{U})}\log{(\frac{r}{\epsilon})} $$ and also
$$(m+n-r)r+ \frac{\frac{2m}{\psi(\mathbb{U})}(r+2 +\log{\frac{1}{\epsilon}})}{\psi(\mathbb{V})}  $$

 Throught the proof we will refer to the proof Theorem 4 - complexity bound regarding $\mathbf{ERRE}$ of the paper \cite{poczos2020optimal} also we are denoting $\psi(\mathbb{U})$ by $k$ to have the consistency with that proof.

\vspace{4mm}
\subsection{$(m+n-r)r+  2 \frac{m }{\psi(\mathbb{U})}\log{(\frac{r}{\epsilon})} $}
\vspace{4mm}

We first start with the case that if minimum of these two quantities is $2 \frac{m }{k}\log{(\frac{r}{\epsilon})}$.
As the matrix has rank $r$, there exists at least one set of linearly independent columns with $r$ columns.
We select the set of linear independent columns$-C$ that has lexicographically smallest indices. 
We show sampling  $2 \frac{m }{k}\log{(\frac{1}{\epsilon})}$ entries from each column will give us the probability of at least $1-r\epsilon$ correctly recovery.\\[2ex]
From the step 5 of proof of the theorem $\mathbf{ERRE}$ we can see sampling $2 \frac{m }{k}\log{(\frac{1}{\epsilon})}$ entries from an active column, would give guarantee of probability of at least $1-\epsilon$ detection of independence.
As $C$ is lexicographically smallest, each column is active on the time entries sampled from it, and each of $r$ columns will succeed with probability at least $1-\epsilon$.
Therefore, using union bound,  with probability $1-r\epsilon$ all of the columns in C will succeed, which guarantees the exact recovery. \\[2ex]
Replacing $\epsilon$ by $\frac{\epsilon}{r}$ will conclude the result that sampling $2 \frac{m }{k}\log{(\frac{r}{\epsilon})}$  from each column will guarantees the correctness of the algorithm with probability at least $1-\epsilon$.

\vspace{4mm}

\subsection{  $(m+n-r)r+ \frac{\frac{2m}{\psi(\mathbb{U})}(r+2 +\log{\frac{1}{\epsilon}})}{\psi(\mathbb{V})}  $.
  }
\vspace{4mm}

From the follow up of lemma 13 in \cite{poczos2020optimal}, we conclude that in a process of $\frac{k}{m}$ probability success and $1-\frac{k}{m}$ probability of failure, having 
$\frac{2m}{k}\big(r+2 +\log{\frac{1}{\epsilon}}\big)$ trial is enough to guarantee getting $r$ many success with at least probability $1-\epsilon$.\\[2ex]
Failure probability of the algorithm $\mathbf{EREI}$ is equal to failing finding $r$ linearly independent columns. 
Consider following equation:
\vspace{3mm}
\begin{align*}
    P\big(\mathbf{EREI} fails \big) &= P\Big(\mathbf{EREI} fails \text{ and TNAO}  \geq \frac{2m}{k}\big(r+2 +\log{\frac{1}{\epsilon}}\big)\Big) \\[1.5ex]
    &+P\Big(\mathbf{EREI} fails \text{ and TNAO}  < \frac{2m}{k}\big(r+2 +\log{\frac{1}{\epsilon}}\big)\Big) 
\end{align*}
\vspace{3mm}

where we denote TNAO as total number of active observations. 
Recall that we call an observation active, if it is active in the execution time (the column is still not contained in the current column space).
Intuitively we represent NAO by number of active observations executed by the algorithm $\mathbf{EREI}$ in the given specific time.\\[2ex]
From lemma 18, if there exists an active column, then there exists at least $t$ many active columns.
Therefore, failure of the algorithm is equivalent to existence of an active column when algorithm terminates.
Moreover each of our observations in those columns were active observations and considering the fact that we observed  $ \frac{\frac{2m}{k}\big(r+2 +\log{\frac{1}{\epsilon}}\big)}{t}$ many entries in each of them, total number of active observations is at least 

\begin{align*}
t \frac{\frac{2m}{k}\big(r+2 +\log{\frac{1}{\epsilon}}\big)}{t} =    
\frac{2m}{k}\Big(r+2 +\log{\frac{1}{\epsilon}}\Big)
\end{align*}

Therefore 
$P\Big(\text{NAO}  < \frac{2m}{k}\big(r+2 +\log{\frac{1}{\epsilon}}\big) | \mathbf{EREI} fails \Big) = 0$ and using Bayesian rule we conclude 

\begin{align*}
P\Big(\mathbf{EREI} fails \text{ and TNAO}  \leq \frac{2m}{k}\big(r+2 +\log{\frac{1}{\epsilon}}\big)\Big)=0    
\end{align*}

Then, following equation simplly satisfied:

\begin{align*}
    P\big(\mathbf{EREI} fails \big) = P\Big(\mathbf{EREI} fails \text{ and TNAO}  \geq \frac{2m}{k}\big(r+2 +\log{\frac{1}{\epsilon}}\big)\Big) 
\end{align*}

\vspace{3mm}

We can observe the following inequality as $\mathbf{EREI}$  may tamporarily fail at the point that the number of active observations is $\frac{2m}{k}\big(r+2 +\log{\frac{1}{\epsilon}}\big)$ but it can succeed finding remaining independent columns later during the execution:

\vspace{3mm}

\begin{align*}
    P\Big(\mathbf{EREI} \text{ currently fail when NAO =}&\frac{2m}{k}\big(r+2 +\log{\frac{1}{\epsilon}}\big) \Big) \geq\\[2ex]
    &P\Big(\mathbf{EREI} fails \text{ and TNAO}  \geq \frac{2m}{k}\big(r+2 +\log{\frac{1}{\epsilon}}\big)\Big) 
\end{align*}
Therefore we conclude that:

\begin{align*}
    P\big(\mathbf{EREI} fails \big) \leq 
        P\Big(\mathbf{EREI} \text{ currently fail when NAO =}&\frac{2m}{k}\big(r+2 +\log{\frac{1}{\epsilon}}\big) \Big)
\end{align*}

\vspace{3mm}

Remember the fact that at each active observation probability of $\mathbf{EREI}$ detecting linear independence of  is larger or equal than $\frac{k}{m}$. 
From the previous discussion if the probability is exactly equal to $\frac{k}{m}$ then still not finding $r$ linearly independent column at $\frac{2m}{k}\big(r+2 +\log{\frac{1}{\epsilon}}\big)$ observations is less than $\epsilon$.
Therefore, $\mathbf{EREI}$ not detecting r linearly independent column after $\frac{2m}{k}\big(r+2 +\log{\frac{1}{\epsilon}}\big)$ observations is smaller than $\epsilon$, which is equivalent to 

\begin{align*}
    P\big(\mathbf{EREI} fails \big) \leq  \epsilon
\end{align*} 

as desired.
\end{proof}

\vspace{5mm}

\begin{corollary}
Lets assume that we have estimated values of rank $r$, column space coherence number $\mu_0$ and estimated row space \textit{sparsity-number} is $\psi(\mathbb{V})$. Then, if $\psi(\mathbb{V})$ is $\mathcal{O}(1)$ then observation complexity is buounded by $(m+n-r)r + \mathcal{O}(nr\mu_0 \log{\frac{r}{\epsilon}}) $ and if $\psi(\mathbb{V})$ is $\mathcal{O}(n)$ then with probabiliy $\frac{1}{2^{\mathcal{O}(r)}}$ the bound is  $\mathcal{O}((m+n-r)r)$ which is theoretical lower bound for exact completion problem.  
\end{corollary}

\newpage

\bibliography{collas2022_conference}
\bibliographystyle{collas2022_conference}

\appendix

\end{document}

%% file: math_commands.tex
\usepackage{amsmath,amsfonts,bm}

\def\eqref#1{equation~\ref{#1}}

\def\1{\bm{1}}

\DeclareMathAlphabet{\mathsfit}{\encodingdefault}{\sfdefault}{m}{sl}
\SetMathAlphabet{\mathsfit}{bold}{\encodingdefault}{\sfdefault}{bx}{n}

